\documentclass{article}

\usepackage[preprint]{style}
\usepackage{natbib} 
\usepackage[algo2e,ruled,linesnumbered,noend]{algorithm2e}
\SetKwInput{KwInput}{Input}                
\SetKwInput{KwOutput}{Output}              
\usepackage{adjustbox}
\usepackage{algpseudocode}
\usepackage[format=plain]{caption}
\usepackage{algorithm}
\usepackage{algorithmicx}
\usepackage{graphicx}
\usepackage{caption}
\usepackage{subcaption}
\usepackage[table,xcdraw]{xcolor}
\usepackage{multicol}
\usepackage{booktabs}
\usepackage{tikz}
\usepackage{bm}
\usepackage{enumitem}
\setlist{nosep}
\usetikzlibrary{shapes.geometric, arrows}
\usepackage[title]{appendix}
\usepackage{listings}

\usepackage[utf8]{inputenc} 
\usepackage[T1]{fontenc}    
\usepackage{hyperref}       
\usepackage{url}            
\usepackage{booktabs}       
\usepackage{amsfonts}       
\usepackage{amsmath} 
\usepackage{nicefrac}       
\usepackage{microtype}      
\usepackage{xcolor}         
\usepackage[english]{babel}
\usepackage{amsthm}
\newtheorem{theorem}{Theorem}

\newtheorem{definition}{Definition}

\tikzstyle{model} = [rectangle, rounded corners, minimum width=1.1cm, minimum height=0.8cm,text centered, draw=black, fill=red!30]
\tikzstyle{vector} = [rectangle, minimum width=0.5cm, minimum height=0.8cm,text centered, draw=black, fill=green!30]
\tikzstyle{var} = [circle, rounded corners, minimum width=0.8cm, minimum height=0.8cm,text centered, draw=black, fill=gray!10]

\title{Contrastive representations of high-dimensional, structured treatments}

\author{
Oriol Corcoll\\
  Spotify\\
\texttt{oriolc@spotify.com} \\
  \And
 Athanasios Vlontzos\\
  Spotify\\
  Imperial College London \\
  \texttt{athanasiosv@spotify.com} \\
  \And
  Michael O'Riordan\\
  Spotify\\
  \texttt{moriordan@spotify.com} \\
   \And
   Ciar\'{a}n M. Gilligan-Lee\\
  Spotify \\
  University College London \\
  \texttt{ciaranl@spotify.com}\\
  \texttt{ciaran.lee@ucl.ac.uk} \\
}

\begin{document}

\maketitle

\begin{abstract}
Estimating causal effects is vital for decision making. In standard causal effect estimation, treatments are usually binary- or continuous-valued. However, in many important real-world settings, treatments can be structured, high-dimensional objects, such as text, video, or audio. This provides a challenge to traditional causal effect estimation. While leveraging the shared structure across different treatments can help generalize to unseen treatments at test time, we show in this paper that using such structure blindly can lead to biased causal effect estimation. We address this challenge by devising a novel contrastive approach to learn a representation of the high-dimensional treatments, and prove that it identifies underlying causal factors and discards non-causally relevant factors. We prove that this treatment representation leads to unbiased estimates of the causal effect, and empirically validate and benchmark our results on synthetic and real-world datasets. 
\end{abstract}

\section{Introduction}

Estimating the causal effect of a treatment is crucial for 
actionable decision-making \cite{richens2020improving, vlontzos2023estimating, gilligan2020causing, Pearl2009, jeunen2022disentangling, van2023estimating, corcoll2022did, reynaud2022d, zeitler2023non, o2024spillover, van2023estimating} In standard effect estimation, treatments are usually binary- or continuous-valued. However, in many cases of real-world importance, treatments correspond to complex, structured, high-dimensional objects, such as text, audio, images, graphs, or products in an online market place, to name a few. This setting provides a challenge to traditional causal estimation methods that must be overcome if we are to understand cause and effect in real-world settings. 

While leveraging shared structure across different treatments can help generalize to unseen treatments at test time, and improve data efficiency, we show in this paper that using such structure blindly leads to biased causal effect estimation. Indeed, in most cases, the outcome is actually caused by underlying causal variables corresponding to latent aspects of the complex treatment object we observe---the tone of a text, for instance. The object we use to characterise the treatment can be thought of as a high-dimensional proxy for these underlying causal latent variables. Importantly, the treatment object can also be proxies for other latent variables which do not causally impact the outcome, such as the style of a text. We show that when such non-causal latent variables are correlated with confounding variables in a given setting, then directly using the high-dimensional, structured treatment for causal effect estimation leads to bias---even when all confounders are observed. 

Consider the example of estimating the impact of a product review on sales of that product. Here, the positive or negative tone of the review will likely be the main driver of impact to sales. To estimate the effect, however, all we have access to is the full text of the review. Other latent aspects of the text---such as style---may not impact sales, yet are mixed together with the tone of the message in the text itself. In this paper we show that if these non-causal latents are correlated with any relevant confounders in this context---the writers affinity for the company selling the product, for example---this can lead to bias when directly using the text as the treatment in effect estimation. 


We address this challenge by devising a novel contrastive approach that learns a representation of the high-dimensional treatments which provably identifies the relevant causal latents and discards non-causal ones. We prove that using this representation as the treatment in causal effect estimation leads to an unbiased estimates of causal effects. Having such a \emph{causally} relevant representation for high-dimensional treatments has utility beyond estimating causal effects. Indeed, if we can understand the causal components that cause products in an online marketplace to be purchased, we cloud improve product recommendation in that marketplace. Moreover, if we learn the aspects of drug molecule that cause reduction in the severity of symptoms for a given disease, then we could find drugs with similar causal aspects more efficiently, thus potentially improving drug discovery.
Finally, we validate our results on synthetic and real-world data, and empirically demonstrate that previous work on effect estimation with high-dimensional treatments yield biased causal effects.



In summary, our main contributions are:

\begin{enumerate}
    \item A novel contrastive method to learn a causally-relevant representation of complex, high-dimensional, structured treatments.
    \item A proof that our representation identifies the causal latents and discards the non-causal ones.
    \item A proof that using this representation leads to unbiased causal effect estimation.
    \item An empirical validation of our method on simulated and real-world datasets, where we outperform previous methods for causal effect estimation with high-dimensional treatments.
\end{enumerate}

\section{Background and definitions}
\label{sec:background}

We adopt the Structural Causal Model (SCM) framework as introduced by \cite{Pearl2009}. 

\begin{definition}[Structural Causal Model] \label{functional causal model}
\label{scmdef}
A structural causal model (SCM) specifies a set of latent variables $U=\{u_1,\dots,u_n\}$ distributed as $P(U)$, a set of observable variables $X=\{X_1,\dots, X_m\}$, a directed acyclic graph (DAG) $G$, called the \emph{causal structure} of the model, whose nodes are the variables $U\cup X$, a collection of functions $F=\{f_1,\dots, f_n\}$, such that $X_i = f_i(\text{PA}(X_i), u_i), \text{ for } i=1,\dots, n,$ where $\text{PA}$ denotes the parent observed nodes of an observed variable.
\end{definition}

A (hard) intervention on variable $T$ is denoted by $\text{do}(T=t)$, and it corresponds to removing all incoming edges in the causal graph and replacing its structural equation with a constant. 

The main causal quantity of interest in this work is the conditional average treatment effect (CATE), which corresponding to the change in outcome for different treatments $T,T'$ at covariate value $x$:

$$\tau(T,T',x) := \mathbb{E}\left(Y \mid \text{do}(T), X=x\right) - \mathbb{E}\left(Y \mid \text{do}(T'), X=x\right)$$

When confounders are observed and d-separate treatment and outcome, the CATE can be estimated via \emph{back-door} adjustment \cite{Pearl2009} as follows:

$$\tau(T,T',x) = \mathbb{E}\left(Y \mid T, X=x\right) - \mathbb{E}\left(Y \mid T', X=x\right)$$

\section{The problem}
\label{sec:problem}

\begin{figure}[!ht] 
\centering
    \includegraphics[scale=0.25]{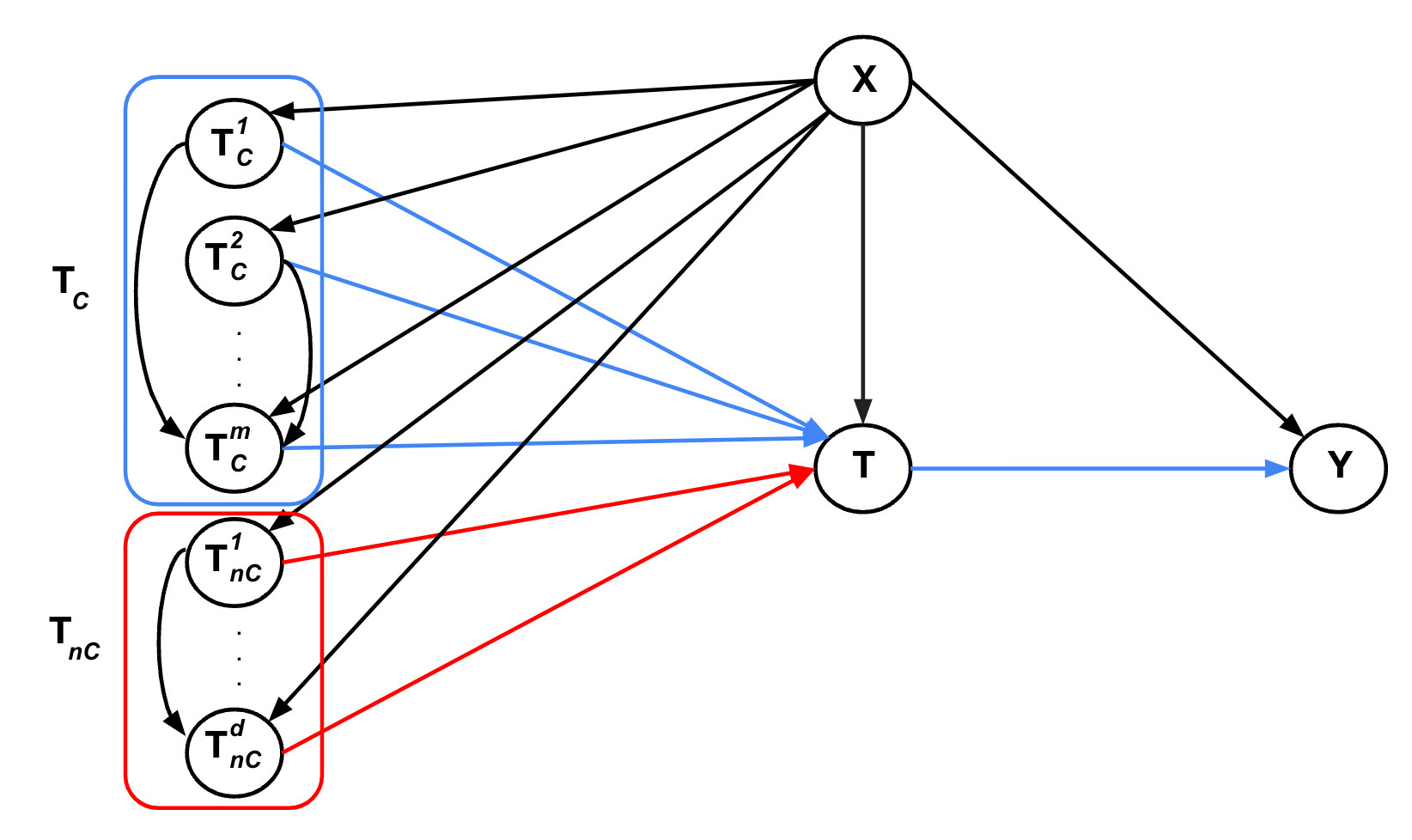}
    \caption{DAG for our problem. While $T$ depends on both $T_C$ and $T_{nC}$ in the structural equations, $T=m(T_C,T_{nC})$ outcome $Y$ only depends on $T$ through $T_C$: $Y=f(T_C, X, \epsilon_Y).$ This is represented graphically by the \emph{blue} arrow from $T_C$ to $T$, and on to $Y$, while the arrow from $T_{nC}$ to $T$ is \emph{red}.}
    \label{fig:DAG}
\end{figure}

In this paper, we consider a setting where the object describing the treatment is generated by some collection of latent variables which can causally interact with one another. These could be, for instance, latent aspects of a piece of text, such as tone or style, a collection features representing a video, or the structure of the bonds in a molecule. We denote the causally relevant latent variables by $T_C=\{T_C^1,\dots,T_C^m\}$ and the non-causally relevant latent variables by $T_{nC}=\{T_{nC}^1,\dots,T_{nC}^d\}$. In the general case, we may not be given direct access to the latents themselves, but some function of them. That is, the treatment we're given for a specific problem, $T$, corresponds to a (potentially non-linear) mixture of these latents, $T=m(T_C, T_{nC})$. The following structural equations categorise the causal relationships between the treatment, $T$, the confounders $X$, and the outcome $Y$: $X=l(\epsilon_X), T_C=g(X, \epsilon_{T_C})$, $T_{nC}=h(X, \epsilon_{T_{nC}})$, $T=f(T_C, T_{nC})$ and outcome $Y=f(T_C, X, \epsilon_Y)$, where noise terms are drawn i.i.d. $\epsilon_i\sim P(\epsilon_i)$. The DAG for this setting is shown in Figure~\ref{fig:DAG}.

In this setting, a given $T_C$ is mapped to a set of $T$ values, indexed by the $T_{nC}$ latents: $T_C\rightarrow\mathcal{T}_{T_C}=\{T=f(T_C, T_{nC})\}_{T_{nC}}$. For causal effect estimation using treatment $T$ to be unbiased, we require that the CATE using $T$ must reproduce the correct CATE with $T_C$. That is: 

$$\int\left(\tau(T_C, T_C', X) - \tau(T, T', X) \right)^2 P(X)dX = 0, \quad\forall T\in\mathcal{T}_{T_C} \text{ and } T'\in\mathcal{T}_{T_C'}.$$

The standard approach to estimating the causal effect of a treatment, $T$, on outcome, $Y$, with confounders, $X$, is to estimate $Y$ given $T$ and $X$---known as back-door adjustment \cite{Pearl2009}. We now show that this approach can result in an unbiased estimate of the causal effect when the treatment $T$ is a mixture of causal and non-causal latents.

\begin{theorem}\label{theorem:back door}
Consider treatment $T$ from the DAG of Figure~\ref{fig:DAG}, with structural equations as given at the start of Section~\ref{sec:problem}. Back-door adjustment directly using $T$ leads to biased causal effect estimation. 
\end{theorem}

\begin{proof}
To show back-door adjustment with $T$ does not suffice for causal effect estimation, we just need to construct at least one example where it fails. Consider the following data generation process:
$$X=\epsilon_X,  \text{ } T_C=\alpha X + \epsilon_{T_C},  \text{ } T_{nC}=\beta X + \epsilon_{T_{nC}},  \text{ } Y=\rho T_C + \delta X +\epsilon_Y, \text{ and } T=\begin{bmatrix}
T_C  \\
T_{nC}
\end{bmatrix}$$
with $\epsilon_i \sim \mathcal{N}(0,\sigma_i)$. This provides us with a joint distribution $P(Y, X, T).$

In order to show backdoor adjustment fails, we need to show that regressing $Y$ onto $T$ and $X$ does not always result in unbiased estimates of the causal effect of $T_C$ on $Y$. That is, we need to show the existence of a model where $\mathbb{E}(Y \mid T, X)$ is equal to $\mathbb{E}(Y \mid T_C, X)$ from the above data generation process, but where $\tau(T, T', X)$ from this model does not equal $\tau(T_C, T_C', X)$ from the data generation process, above, where $T\in\mathcal{T}_{T_C}$ and $T'\in\mathcal{T}_{T_{C}'}$. In this case, $\mathbb{E}(Y \mid T, X)$ does not identify $\mathbb{E}(Y \mid \text{do}(T_C), X)$. 


Consider the following model for $Y$: 
$$Y=\begin{bmatrix}
\alpha & \frac{\sigma}{2\beta} 
\end{bmatrix}\begin{bmatrix}
T_C \\
T_{nC}
\end{bmatrix} + \frac{\sigma}{2}X +\epsilon_Y = \begin{bmatrix}\alpha & \frac{\sigma}{2\beta} 
\end{bmatrix} T + \frac{\sigma}{2}X +\epsilon_Y$$
As $T_{nC}=\beta X + \epsilon_{T_{nC}}$ it follows that the expected values for $Y$ given $T$ and $X$, $\mathbb{E}(Y \mid T, X),$ generated by this model is equivalent to $\mathbb{E}(Y \mid T_C, X)$ from the data generation process, above. As $\mathbb{E}(Y \mid T, X) = \mathbb{E}(Y \mid T_C, X)$, this model is a possible solution to regressing $Y$ on $T$ and $X$. However, when we intervene on $T_{nC}$, we break the relationship between $X$ and $T_{nC}$, which reveals we have not learned the correct causal model. To see this, consider $T = \begin{bmatrix} T_C \\ T_{nC}\end{bmatrix}$ and $T'=\begin{bmatrix} T_C \\ T_{nC}'\end{bmatrix}$ with $T_{nC}\neq T_{nC}'$. Here, $\tau(T_C, T_C, X)=0$, but $\tau(T, T', X)=\frac{\sigma}{2\beta} (T_{nC}-T_{nC}')\neq 0.$
\end{proof}

Why does this happen? Well the $T_{nC}$ are proxies for the confounders, $X$. So using them in our estimation can make it look like we have suitably controlled for the confounders using the $T_{nC}$, but when we intervene on the $T_{nC}$ we break the link between $T_{nC}$ and $X$ and reveal that we have not appropriately controlled for the true confounders.

In the experiments section we investigate this empirically  on both synthetic and real data, and find that back door adjustment of $T$ leads to biased effect estimation.

In order to estimate an unbiased causal effect, we should not directly use the high-dimensional treatment itself. Instead, in the example from Theorem~\ref{theorem:back door}, had we been able to use a representation of $T$ that does not contain any information about the non-causal latents, then the backdoor adjustment with this representation would have identified the correct causal effect. This motivates using a representation of the treatment $\psi(T)$ to estimate causal effects. If we learn a representation of $T$ that doesn't depend on the non-causal latents, it seems intuitive that effect estimation will be unbiased in general. We now prove this is necessary and sufficient. 

\begin{theorem}
Causal effect estimation is unbiased if and only if a representation of $T$ is used that contains no information about the non-causal latents.    
\end{theorem}

\begin{proof}
Assume first that $\psi(T)$ contains no information about $T_{nC}$. Then it must map all $T\in\mathcal{T}_{T_C}$ to the same value. That is, $\psi(T)$ is just a reparametrization of $T_C$, as $T_C$ are in one-to-one correspondence with $\mathcal{T}_{T_C}$. This moreover means that $\psi(.)$ preserves interventions on $T_C$, which implies it preserves CATE. 

To show the other direction, that unbiased CATE implies $\psi(T)$ contains no information about $T_{nC}$, consider the following. For the CATE to be unbiased we require that
$$\int\left(\tau(T_C, T_C', X) - \tau(T, T', X) \right)^2 P(X)dX = 0, \quad\forall T\in\mathcal{T}_{T_C} \text{ and } T'\in\mathcal{T}_{T_C'}.$$

In particular, this means that for $T,T'\in\mathcal{T}_{T_C}$ with $T\neq T'$, we have that:

$$ 
\begin{aligned}
0&=\int\left(\tau(T_C, T_C, X) - \tau(T, T', X) \right)^2 P(X)dX \\
&=\int\left(\tau(T, T', X) \right)^2 P(X)dX 
\end{aligned}
$$
As all terms in the integral are positive, for it to be equal to zero we must have each term equal to zero. But $P(X)$ is positive on its support set, hence we have that for all $X$
$$ 
\begin{aligned}
0 &= \tau(T, T', X)
= \mathbb{E}(Y | \psi(T),X) - \mathbb{E}(Y |\psi(T'),X) \\
&\implies \mathbb{E}(Y | \psi(T),X) = \mathbb{E}(Y |\psi(T'),X), \text{  } \forall X \text{ and } T,T'\in\mathcal{T}_{T_C}
\end{aligned}
$$
This tells us that $\psi(T)$ and $\psi(T')$ are interventionally equivalent from the point of view of $Y$. As the only difference between $T$ and $T'$ are their non-causal latents, then the representation $\psi(.)$ must map all $T\in\mathcal{T}_{T_C}$ to the same value. Hence it must disregard information about $T_{nC}.$ 
\end{proof}

\section{A contrastive algorithm for learning causally relevant treatment representations}
\label{sec:method}

How do we learn a representation that has no information about $T_{nC}$? To build intuition, consider the structural equations from Section~\ref{sec:problem}, with $f(.)$ in $Y=f(T_C, X)$ an invertible function. Suppose we have two data points where the $X$ and $Y$ values are the same, but the $T$'s are different: $[T,x,y], [T',x,y]$. As $Y=f(T_C, X)$, we have that $f(T_C, C)=f(T_C', X)$. As this function is invertible, the causal components of $T$ and $T'$ are the same. However, data points where the $X$ values are the same, but the $Y$ values are different must have different causal components. 

This observation suggests a contrastive algorithm with positive and negative pairs as below should push $T$ with similar $T_C$ together, and different $T_C$ apart \cite{oord2018representation,tingey2021contrastive}.

\textbf{Positive pairs}: $[T,x,y], [T',x,y]$ such that $T \neq T', \text{ and } X = X', \text{ and } Y = Y'$ \\
\textbf{Negative pairs}: $[T,x,y], [T',x,y]$ such that $T \neq T', \text{ and } X = X', \text{ and } Y \neq Y'$

We now show this provably \emph{block identifies} the causal components of $T$. The resulting representation $\psi(T)$ contains all and only information about $T_C$: there exists an invertible $\phi$: $\psi(T)=\phi(T_C)$.

\begin{theorem}
Assume a structural causal model with DAG from Figure~\ref{fig:DAG} and equations $X=l(\epsilon_X), T_C=g(X, \epsilon_{T_C})$, $T_{nC}=h(X, \epsilon_{T_{nC}})$, $T=m(T_C, T_{nC})$ and $Y=f(T_C, X)$, with all functions smooth and invertible with smooth inverses, with noise terms drawn i.i.d. $\epsilon_i\sim P(\epsilon_i)$ from smooth distributions that have $P(\epsilon_i)>0$ almost everywhere. Then the contrastive approach outlined above yields a representation of $T$ that block-identifies the causal latents. 
\end{theorem}

\begin{proof}
Theorem 4.2 from \cite{von2021self} can be applied to help us prove that our contrastive learning approach with positive and negative pairs as above yields a treatment representation that identifies the latents in $T$ that $Y$ causally depends on. This Theorem states that when we have data involving two classes of variables, if we can create pairs of data points with one of the pair being the original view 
and the other an augmented view, such that a subset of one class is different to the original view, then we can block identify the class of variables that remains the same. The theorem holds as long as the underlying data generating process consists of smooth, invertible functions with smooth inverses, and smooth distributions that are non-zero almost everywhere.

We are going to use the above described theorems to prove we can block identify $T_C$. To do this, we need to show that we can take an observation $T=(T_C , T_{nC})$ and ``augment'' it to get $(T_C , T_{nC}')$, where $T_C$ is the same but (possibly some subset of) $T_{nC}$ is not.

Consider two data points where the $X$ and $Y$ values are the same, but the $T$'s are possibly different: $[T,x,y], [T',x,y]$. We have that $y=f(T_C, x) = f(T_C', x)$. As $f$ is invertible we have that $T_C=T_C'$. What this means is that the causally relevant components of $T$ and $T'$ are the same when the values of $Y$ and $X$ are the same. But we need to also show that our augmentations have different $T_{nC}$ components. That is, these augmentations leave $T_C$  invariant, but change (some subset of) $T_{nC}$. 

If there exists different $T, T'$ that occur with the same values of $X$ and $Y$, then $T_{nC}$ must be different, as $T_C$ is the same. But does there exist at least two different $T$’s for some values of $X$ and $Y$? If there doesn't then this means that $T_{nC}$ only depends on $X$ and not the noise term $\epsilon_{T_{nC}}$, which is a contradiction as we assumed at the start that $P(\epsilon_{T_{nC}})$ has non-trivial support. Thus choosing data augmentations in this fashion ensures $T_C$ is invariant between augmentations, but $T_{nC}$ is not. Hence we can apply Theorem 4.1 from \cite{von2021self}to conclude the proof.
\end{proof}

Given high-dimensional covariates $X$ and continuous outcome $Y$, the approach to constructing positive pairs from the start of this section is impractical. Instead of demanding equality $X=X'$ and $Y=Y'$ between samples to find positive pairs, one could instead impose thresholds $\delta, \epsilon$ and consider $X,X'$ and $Y,Y'$ ``close'' if $|X-X'| \leq \delta$ and $|Y-Y'|\leq \epsilon$. Additionally, one could also first learn a low-dimensional representation $g(.)$ of $X$ and consider $X,X'$ close if $|g(X)-g(X')| \leq \delta$. Indeed, for continuous $g(.)$, if $g(X),g(X')$ are close, so too are $X,X'$. In this setting, samples $[X, T, Y], [X', T', Y']$ with: $|g(X)-g(X')| \leq \delta$ and $|Y-Y'|\leq \epsilon$ also have similar $T_C$ and $T_C'$. Indeed we have $|f(T_C, X)-f(T_C',X')|=|Y-Y'|\leq \epsilon$. For continuous $g(.)$ with $|g(X)-g(X')| \leq \delta$, there exists a $\rho$ such that we have $X\approx X'+\rho$. $f(T_C',X')=f(T_C',X+\rho)\approx f(T_C',X)$ by Taylor expanding smooth $f(.)$ with small $\rho$. This implies $|f(T_C, X)-f(T_C',X)|\epsilon, \forall X$. For smooth $f(.)$ we have that $T_C$ and $T_C'$ are close. If, instead we had $|Y-Y'|>\epsilon$, then $T_C$ and $T_C'$ would not be close, which provide negative samples. Hence a contrastive approach with such positive and negative pairs should still intuitively push $T$'s with similar $T_C$'s together, and dissimilar $T_C$'s apart. 

\begin{algorithm2e}[t]
 \DontPrintSemicolon
    \KwInput{Dataset $\{(x, t, y)\}$, representation network $\Phi_{\theta}$ parameterised by $\theta$, representation network $g(.)$ and threshold $\delta$ for $X$, effect threshold $\epsilon$ for $Y$.}
\KwOutput{Learned representation network $\Phi_{\theta}$.}
\For{ each sample $(x, t, y)$:}{
construct \textbf{positive pair}: $(x', t', y')$ with similar $x,x'$: $|g(x)-g(x')| \leq \delta$, and similar $y, y'$: $|y-y'|\leq \epsilon$ \\
construct \textbf{negative pair}: $(x', t', y')$) with similar $x,x'$: $|g(x)-g(x')| \leq \delta$, and dissimilar $y, y'$: $|y-y'|> \epsilon$}
Compute contrastive loss with these positive and negative pairs and update $\theta$ using SGD \\
\Return{$\Phi_{\theta}$}
\caption{Practical contrastive algorithm for representations of high-dimensional treatments.}
\label{alg:constrastive}
\end{algorithm2e}

Algorithm~\ref{alg:constrastive} describes this practical contrastive approach to learn representations of high-dimensional treatments, which we empirically validate on synthetic and real data in Section~\ref{sec:experiments}

\section{Related work}
\label{sec:related}
While many works have approached the task of invariance and disentanglement through contrastive learning (for example \cite{NEURIPS2021_97416ac0,NEURIPS2021_b6cda17a,10415220}), few, to the best of our knowledge, have looked into disentangling causal and non-causal components for causal outcome estimation.
Previous work has investigated the estimation of causal effects from high-dimensional treatments. \cite{kaddour2021causal} introduced `Structured Intervention Networks,' an approach which uses representation learning and alternating gradient descent in a semi-parametric model, known as the generalised Robinson decomposition, to estimate causal effects from high-dimensional, structured treatments. \cite{nabi2022semiparametric} also explores a semi-parametric approach to the problem, casting the problem as a generalisation of sufficient dimension reduction using influence functions in order to estimate the causal effect. The authors only consider the average treatment effect (ATE) as opposed to the more general conditional average treatment effect (CATE). \cite{harada2021graphite} introduced `GraphITE,' a method for estimating causal effects when treatments are graphs. To learn a representation of the treatment, the authors use a regularization term using the Hilbert-Schmidt Information Criterion (HSIC)---which introduces high computational cost---but do not prove it correctly identifies relevant causal latents. Finally, ~\cite{pryzant2020causal} investigate the causal impact of text attributes, where treatments can be mixtures of causal and non-causal latents from a text. Given some assumptions they bound the bias in causal effect estimation using the full text, but they don't provide way to remove non-causal latents, as we do here.
In contrast to the above work, our approach is fully non-parametric and provides theoretical guarantees that we have correctly identified the causally relevant latents.

\section{Experiments}
\label{sec:experiments}
The contrastive approach presented in Section \ref{sec:method} aims to make a effect estimation model more robust to non-causal information present in high-dimensional treatments. Non-causal information presents a crucial risk to machine learning models.
A model fails to discard non-causal information due to two types of errors: irreducible and reducible errors. In other words, error due imperfect information among covariates and treatment; or due to the inability of the learning mechanism to model the problem correctly.
Regardless of the error type, causal models should discard non-causal information.


\textbf{Datasets}
A common characteristic among each of the datasets used in the experiments is that, similarly to Fig. \ref{fig:DAG}, multidimensional treatments are constructed from causal and non-causal information. The goal is to evaluate that the model is able to discard non-causal information and retain causal information.
To add complexity through irreducible error, we use a \textbf{Synthetic dataset}, as in Fig. \ref{fig:DAG}. This synthetic dataset has 1000 samples (70\% for training and 30\% for evaluation); the treatment variable has 10 dimensions, 5 are causal and 5 are non-causal, both highly correlated with the covariates; the outcome is causally determined by the covariates, the causal part of the treatment and random noise.
On the other hand, to introduce complexity through reducible error we use the \textbf{Molecule dataset} \cite{ramakrishnan2014quantum, weinstein2013cancer} and the \textbf{Coat recommender dataset} \cite{pmlr-v48-schnabel16}. These two datasets have more complex causal relations than the Synthetic dataset and large part of the error should be reducible by the model.
See Appendix \ref{appendix:dataset}.

\textbf{Models}
The contrastive method is applied to a classical CATE model (see Fig. \ref{fig:model_dag}). The contrastive loss chosen en these experiments is the Triplet loss \cite{schroff2015facenet}. Positive and negative pairs are selected using a simple clustering method, referred as $g$ in Section \ref{sec:method}, where each component of the variable is bucketed, thus converting continuous variables into discrete ones where the method described is easily applicable.
The \textbf{contrastive CATE model} is compared to two baselines; the exact same \textbf{CATE model} without the contrastive method and the \textbf{SIN model} from \cite{kaddour2021causal}.
The CATE model, with and without the contrastive loss, is implemented as a linear model for the first set of experiments. As a representation of the treatment is needed to compute the contrastive loss, this is computed by applying the treatment weights of the model onto the treatment, and the outcome of this operation will be the representation used. For the second set of experiments, the CATE model is implemented as a Neural Network with treatment and covariate branches producing their respective representations.

\textbf{Evaluation metrics}
In order to measure the robustness introduced by the contrastive approach, the first thing to show is that the models have learnt to solve the given problem under no perturbations. For this purpose we include the mean absolute error (MAE) and root mean squared error (RMSE) for unperturbed test data.
Once all models have learnt to solve the task, a key aspect to measure is the ability of a model to ignore non-causal information present in the treatment variable. To this purpose, the experiments below use Precision in Estimation of Heterogeneous Effect (PEHE) \cite{hill2011bayesian}. In essence, PEHE computes the RMSE between the effects of two treatments $(t, t')$
\begin{equation}
    \sqrt{\frac{1}{n}\sum{\biggl[\Bigl(f\left(x, t\right) - f\left(x, t'\right)\Bigl) - \left(y - y'\right)\biggl]^2}},
\end{equation}
where $f$ is the model and $y,y'$ are the true outcomes. When $t$ and $t'$ come from the same $t_C$ but with different $t_{nC}$, this metric provides a measure of how robust the model is to changes in non-causal information, thus measuring the degree to which the model ignores non-causal information.

\subsection{Irreducible error}
A common source of error in ML models is due to the lack of information in their inputs. Problems with imperfect information can make the model rely on correlations instead of the true causal relations between treatment, covariates and outcome.
This experiment aims to study if the contrastive method proposed makes the model more robust to this kind of errors. To this end, the Synthetic dataset is perturbed by adding additional noise to the outcome before training. The added noise comes from a normal distribution with mean zero and its standard deviation increases linearly in steps of $0.1$ starting at $0.0$ up to $1.0$.
Note that due to the learning mechanism used, a model may incorrectly pick on correlations between the treatment and covariates, even without intervening these variables.
\begin{figure}[!ht]
    \centering
    \begin{subfigure}[b]{0.35\textwidth}
        \centering
        \resizebox{0.7\columnwidth}{!}{%
        \begin{tikzpicture}[node distance={20mm}, thick]
            \node[model] (nn) {NN};
            \node[model] (xn) [above right of=nn] {x-NN};
            \node[model] (tn) [above left of=nn] {t-NN};
            \node[var] (x) [above of=xn] {$x$};
            \node[var] (t) [above of=tn]  {$t$};
            \node[var] (y) [below of=nn]  {$y$};
            \draw[->] (x) -- (xn);
            \draw[->] (t) -- (tn);
            \draw[->] (tn) -- node[text width=-0.3cm,midway,above] {$h_t$} (nn);
            \draw[->] (xn) -- node[text width=0.8cm,midway,above] {$h_x$} (nn);
            \draw[->] (nn) -- (y);
        \end{tikzpicture}%
        }
        \caption{CATE model}
        \label{fig:model_dag}
    \end{subfigure}
    \begin{subfigure}[b]{0.4\textwidth}
        \centering
        \includegraphics[scale=0.45]{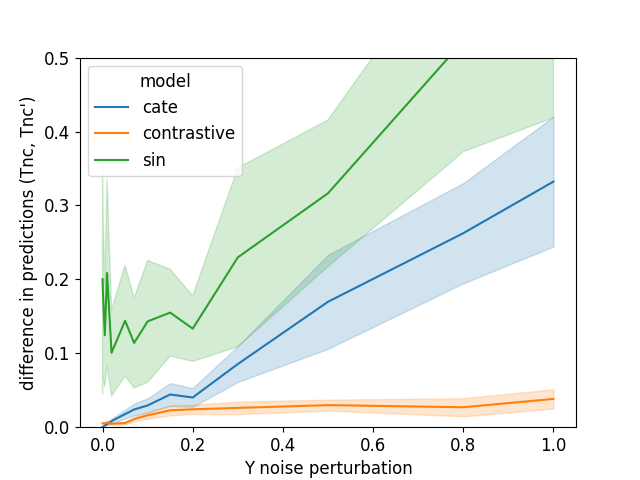}
        \caption{Performance under perturbations to y}
        \label{fig:pehe_contrastive_toy}
    \end{subfigure}
\caption{}
\label{fig:toy}
\end{figure}

An ideal model would predict the same outcome regardless of the non-causal information in the treatment ($t_{nC}$), since this information does not causally influence the outcome.
What the experiment in Fig. \ref{fig:pehe_contrastive_toy} shows is the difference in predictions (the effect) between a sample $(x,t,y)$ and a perturbed version of the treatment $(x,t',y)$ where $t'$ only has its $t_{nC'}$ component changed.
We can see that the contrastive method achieves, to a reasonable degree, that effect; but the CATE and SIN models fail to do so.
Note in Table \ref{table:irreducible} that all models learn to solve the task relatively well\footnote{after extensive hyperparameter search on SIN, it does not achieve similar performance to the CATE models}; nevertheless only the contrastive approach is capable to ignore non-causal information from the treatment.
\begin{table}[!ht]
\centering
\caption{Error metrics with their standard error using 10 different seeds for the Synthetic dataset.}
\label{table:irreducible}
\begin{tabular}{ccccccccc}
    \toprule
    \multicolumn{1}{l}{\textbf{Model}} & \textbf{MAE} & \textbf{RMSE} & \textbf{PEHE}\\ \midrule
    \multicolumn{1}{l}{CATE} & $0.59 \pm 0.8$ & $0.74 \pm 1.0$ & $0.07 \pm 0.1$\\
    \multicolumn{1}{l}{SIN} & $0.9 \pm 0.8$ & $1.16 \pm 0.94$ & $0.20 \pm 0.2$\\ \midrule
    \multicolumn{1}{l}{\textbf{Contrastive}} & $0.63 \pm 0.8$ & $0.78 \pm 1.0$ & $\bm{0.01 \pm 0.0}$\\ \bottomrule
\end{tabular}
\end{table}

Table \ref{table:irreducible} shows quantitatively that the PEHE metric is considerably better in the contrastive case; which is an additional indication that the contrastive model is discarding the non-causal components.

\subsection{Reducible error}
Another source of error is in the intrinsic complexity of the problem, the more complex the problem the harder it is to perform well, requiring larger models or larger datasets. Even when the data has the right information to discard the non-causal information, it may be difficult for a learning mechanism to do so.
This experiment tests the ability of each model to discard non-causal information when the irreducible error does not change (it is intrinsic to the data) but where reducible error is introduced in the form of noise on the non-causal information at test time.
\begin{figure}[!ht]
    \centering
    \begin{subfigure}[b]{0.44\textwidth}
      \centering
      \includegraphics[width=\textwidth]{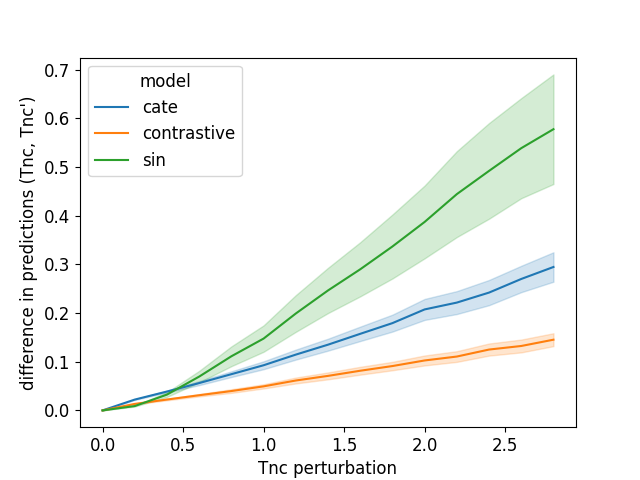}
      \caption{Molecule}
      \label{fig:contrastive_pehe_molecule}
    \end{subfigure}
    \begin{subfigure}[b]{0.44\textwidth}
      \centering
      \includegraphics[width=\textwidth]{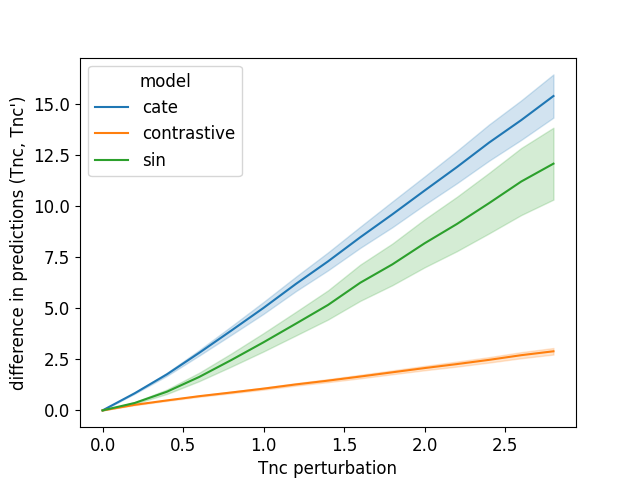}
      \caption{Recommender}
      \label{fig:contrastive_pehe_recommender}
    \end{subfigure}
    \caption{Performance under different perturbations to the non-causal information of the treatment}
    \label{fig:contrastive_pehe}
\end{figure}

\begin{table}[!ht]
\centering
\caption{Error metrics for the Molecule and Recommender datasets. The contrastive approach provides a more robust learning of the causal information than non-contrastive approaches.}
\label{table:reducible}

\begin{tabular}{c|ccc|ccccc}
    \toprule
     & \multicolumn{3}{c}{\textbf{Molecule}} & \multicolumn{3}{c}{\textbf{Recommender}} \\
    \multicolumn{1}{l|}{\textbf{Model}} & \textbf{MAE} & \textbf{RMSE} & \textbf{PEHE} & \textbf{MAE} & \textbf{RMSE} & \textbf{PEHE} \\ \midrule
    \multicolumn{1}{l|}{CATE} & $0.01 \pm 0.0$ & $0.01 \pm 0.0$ & $0.12 \pm 0.1$ & $1.06 \pm 0.0$ & $1.30 \pm 0.0$ & $7.42 \pm 5.2$ \\
    \multicolumn{1}{l|}{SIN} & $0.03 \pm 0.0$ & $0.04 \pm 0.0$ & $0.28 \pm 0.3$ & $1.13 \pm 0.3$ & $1.35 \pm 0.3$ & $5.49 \pm 4.9$ \\ \midrule
    \multicolumn{1}{l|}{\textbf{Contrastive}} & $0.02 \pm 0.0$ & $0.02 \pm 0.0$ & $\bm{0.06 \pm 0.0}$ & $0.99 \pm 0.0$ & $1.2 \pm 0.0$ & $\bm{1.47 \pm 0.9}$ \\ \bottomrule
\end{tabular}
\end{table}

We would expect for the model to be able to completely ignore the non-causal information but as shown in Fig. \ref{fig:contrastive_pehe} the SIN and CATE models are less robust, having larger differences in effect when perturbed the non-causal component of the treatment.
Moreover, Table \ref{table:reducible} again shows how all models have good performance on the problem but that only the contrastive one has low PEHE and thus, it is a more robust method.
This experiment demonstrates that a contrastive loss added to a CATE model promotes learning a causal representation of the treatment.

%

\section{Conclusion}
In this paper we investigated a challenging setting for causal inference with importance for real-world applications: causal effect estimation when treatments are high-dimensional, structured objects. We showed that using the shared structure across different treatments blindly can lead to biased causal effect estimation. To address this challenge we devised a novel contrastive approach that learns a representation of the high-dimensional treatment which provably identifies the underlying causal latents and discards the non-causal ones. We also proved that using this treatment representation provides unbiased causal effect estimation, and empirically validated our results on synthetic and real-world datasets. Lastly, we demonstrated that previous work on causal effect estimation with high-dimensional treatments does not result in unbiased estimation of causal effects. 

\bibliography{ref.bib}

\begin{thebibliography}{25}
\providecommand{\natexlab}[1]{#1}
\providecommand{\url}[1]{\texttt{#1}}
\expandafter\ifx\csname urlstyle\endcsname\relax
  \providecommand{\doi}[1]{doi: #1}\else
  \providecommand{\doi}{doi: \begingroup \urlstyle{rm}\Url}\fi

\bibitem[Corcoll et~al.(2022)Corcoll, Mohamed, and Vicente]{corcoll2022did}
Oriol Corcoll, Youssef Sherif~Mansour Mohamed, and Raul Vicente.
\newblock Did i do that? blame as a means to identify controlled effects in reinforcement learning.
\newblock \emph{Transactions on Machine Learning Research}, 2022.
\newblock ISSN 2835-8856.
\newblock URL \url{https://openreview.net/forum?id=NL2L3XjVFx}.

\bibitem[Gilligan-Lee(2020)]{gilligan2020causing}
Ciar{\'a}n Gilligan-Lee.
\newblock Causing trouble.
\newblock \emph{New Scientist}, 246\penalty0 (3279):\penalty0 32--35, 2020.

\bibitem[Harada and Kashima(2021)]{harada2021graphite}
Shonosuke Harada and Hisashi Kashima.
\newblock Graphite: Estimating individual effects of graph-structured treatments.
\newblock In \emph{Proceedings of the 30th ACM International Conference on Information \& Knowledge Management}, pages 659--668, 2021.

\bibitem[Hill(2011)]{hill2011bayesian}
Jennifer~L Hill.
\newblock Bayesian nonparametric modeling for causal inference.
\newblock \emph{Journal of Computational and Graphical Statistics}, 20\penalty0 (1):\penalty0 217--240, 2011.

\bibitem[Jeunen et~al.(2022)Jeunen, Gilligan-Lee, Mehrotra, and Lalmas]{jeunen2022disentangling}
Olivier Jeunen, Ciar{\'a}n Gilligan-Lee, Rishabh Mehrotra, and Mounia Lalmas.
\newblock Disentangling causal effects from sets of interventions in the presence of unobserved confounders.
\newblock \emph{Advances in Neural Information Processing Systems}, 35:\penalty0 27850--27861, 2022.

\bibitem[Kaddour et~al.(2021)Kaddour, Zhu, Liu, Kusner, and Silva]{kaddour2021causal}
Jean Kaddour, Yuchen Zhu, Qi~Liu, Matt~J Kusner, and Ricardo Silva.
\newblock Causal effect inference for structured treatments.
\newblock \emph{Advances in Neural Information Processing Systems}, 34:\penalty0 24841--24854, 2021.

\bibitem[Li et~al.(2021)Li, Wang, Zhang, Yuan, Li, and Zhu]{NEURIPS2021_b6cda17a}
Haoyang Li, Xin Wang, Ziwei Zhang, Zehuan Yuan, Hang Li, and Wenwu Zhu.
\newblock Disentangled contrastive learning on graphs.
\newblock In M.~Ranzato, A.~Beygelzimer, Y.~Dauphin, P.S. Liang, and J.~Wortman Vaughan, editors, \emph{Advances in Neural Information Processing Systems}, volume~34, pages 21872--21884. Curran Associates, Inc., 2021.
\newblock URL \url{https://proceedings.neurips.cc/paper_files/paper/2021/file/b6cda17abb967ed28ec9610137aa45f7-Paper.pdf}.

\bibitem[Liu et~al.(2024)Liu, Sanchez, Thermos, O’Neil, and Tsaftaris]{10415220}
Xiao Liu, Pedro Sanchez, Spyridon Thermos, Alison~Q. O’Neil, and Sotirios~A. Tsaftaris.
\newblock Compositionally equivariant representation learning.
\newblock \emph{IEEE Transactions on Medical Imaging}, pages 1--1, 2024.
\newblock \doi{10.1109/TMI.2024.3358955}.

\bibitem[Nabi et~al.(2022)Nabi, McNutt, and Shpitser]{nabi2022semiparametric}
Razieh Nabi, Todd McNutt, and Ilya Shpitser.
\newblock Semiparametric causal sufficient dimension reduction of multidimensional treatments.
\newblock In \emph{Uncertainty in Artificial Intelligence}, pages 1445--1455. PMLR, 2022.

\bibitem[Oord et~al.(2018)Oord, Li, and Vinyals]{oord2018representation}
Aaron van~den Oord, Yazhe Li, and Oriol Vinyals.
\newblock Representation learning with contrastive predictive coding.
\newblock \emph{arXiv preprint arXiv:1807.03748}, 2018.

\bibitem[O'Riordan and Gilligan-Lee(2024)]{o2024spillover}
Michael O'Riordan and Ciar{\'a}n~M Gilligan-Lee.
\newblock Spillover detection for donor selection in synthetic control models.
\newblock \emph{arXiv preprint arXiv:2406.11399}, 2024.

\bibitem[Pearl(2009)]{Pearl2009}
Judea Pearl.
\newblock \emph{Causality (2nd edition)}.
\newblock Cambridge University Press, 2009.

\bibitem[Pryzant et~al.(2020)Pryzant, Card, Jurafsky, Veitch, and Sridhar]{pryzant2020causal}
Reid Pryzant, Dallas Card, Dan Jurafsky, Victor Veitch, and Dhanya Sridhar.
\newblock Causal effects of linguistic properties.
\newblock \emph{arXiv preprint arXiv:2010.12919}, 2020.

\bibitem[Ramakrishnan et~al.(2014)Ramakrishnan, Dral, Rupp, and Von~Lilienfeld]{ramakrishnan2014quantum}
Raghunathan Ramakrishnan, Pavlo~O Dral, Matthias Rupp, and O~Anatole Von~Lilienfeld.
\newblock Quantum chemistry structures and properties of 134 kilo molecules.
\newblock \emph{Scientific data}, 1\penalty0 (1):\penalty0 1--7, 2014.

\bibitem[Reynaud et~al.(2022)Reynaud, Vlontzos, Dombrowski, Gilligan~Lee, Beqiri, Leeson, and Kainz]{reynaud2022d}
Hadrien Reynaud, Athanasios Vlontzos, Mischa Dombrowski, Ciar{\'a}n Gilligan~Lee, Arian Beqiri, Paul Leeson, and Bernhard Kainz.
\newblock D’artagnan: Counterfactual video generation.
\newblock In \emph{International Conference on Medical Image Computing and Computer-Assisted Intervention}, pages 599--609. Springer, 2022.

\bibitem[Richens et~al.(2020)Richens, Lee, and Johri]{richens2020improving}
Jonathan~G Richens, Ciar{\'a}n~M Lee, and Saurabh Johri.
\newblock Improving the accuracy of medical diagnosis with causal machine learning.
\newblock \emph{Nature communications}, 11\penalty0 (1):\penalty0 3923, 2020.

\bibitem[Schnabel et~al.(2016)Schnabel, Swaminathan, Singh, Chandak, and Joachims]{pmlr-v48-schnabel16}
Tobias Schnabel, Adith Swaminathan, Ashudeep Singh, Navin Chandak, and Thorsten Joachims.
\newblock Recommendations as treatments: Debiasing learning and evaluation.
\newblock In Maria~Florina Balcan and Kilian~Q. Weinberger, editors, \emph{Proceedings of The 33rd International Conference on Machine Learning}, volume~48 of \emph{Proceedings of Machine Learning Research}, pages 1670--1679, New York, New York, USA, 20--22 Jun 2016. PMLR.
\newblock URL \url{https://proceedings.mlr.press/v48/schnabel16.html}.

\bibitem[Schroff et~al.(2015)Schroff, Kalenichenko, and Philbin]{schroff2015facenet}
Florian Schroff, Dmitry Kalenichenko, and James Philbin.
\newblock Facenet: A unified embedding for face recognition and clustering.
\newblock In \emph{Proceedings of the IEEE conference on computer vision and pattern recognition}, pages 815--823, 2015.

\bibitem[Tingey et~al.(2021)Tingey, Gilligan-Lee, and Dai]{tingey2021contrastive}
Josh Tingey, Ciar{\'a}n~Mark Gilligan-Lee, and Zhenwen Dai.
\newblock Contrastive embedding of structured space for bayesian optimization.
\newblock In \emph{Fifth Workshop on Meta-Learning at the Conference on Neural Information Processing Systems}, 2021.

\bibitem[Van~Goffrier et~al.(2023)Van~Goffrier, Maystre, and Gilligan-Lee]{van2023estimating}
Graham Van~Goffrier, Lucas Maystre, and Ciar{\'a}n~Mark Gilligan-Lee.
\newblock Estimating long-term causal effects from short-term experiments and long-term observational data with unobserved confounding.
\newblock In \emph{Conference on Causal Learning and Reasoning}, pages 791--813. PMLR, 2023.

\bibitem[Vlontzos et~al.(2023)Vlontzos, Kainz, and Gilligan-Lee]{vlontzos2023estimating}
Athanasios Vlontzos, Bernhard Kainz, and Ciar{\'a}n~M Gilligan-Lee.
\newblock Estimating categorical counterfactuals via deep twin networks.
\newblock \emph{Nature Machine Intelligence}, 5\penalty0 (2):\penalty0 159--168, 2023.

\bibitem[Von~K{\"u}gelgen et~al.(2021)Von~K{\"u}gelgen, Sharma, Gresele, Brendel, Sch{\"o}lkopf, Besserve, and Locatello]{von2021self}
Julius Von~K{\"u}gelgen, Yash Sharma, Luigi Gresele, Wieland Brendel, Bernhard Sch{\"o}lkopf, Michel Besserve, and Francesco Locatello.
\newblock Self-supervised learning with data augmentations provably isolates content from style.
\newblock \emph{Advances in neural information processing systems}, 34:\penalty0 16451--16467, 2021.

\bibitem[Wang et~al.(2021)Wang, Yue, Huang, Sun, and Zhang]{NEURIPS2021_97416ac0}
Tan Wang, Zhongqi Yue, Jianqiang Huang, Qianru Sun, and Hanwang Zhang.
\newblock Self-supervised learning disentangled group representation as feature.
\newblock In M.~Ranzato, A.~Beygelzimer, Y.~Dauphin, P.S. Liang, and J.~Wortman Vaughan, editors, \emph{Advances in Neural Information Processing Systems}, volume~34, pages 18225--18240. Curran Associates, Inc., 2021.
\newblock URL \url{https://proceedings.neurips.cc/paper_files/paper/2021/file/97416ac0f58056947e2eb5d5d253d4f2-Paper.pdf}.

\bibitem[Weinstein et~al.(2013)Weinstein, Collisson, Mills, Shaw, Ozenberger, Ellrott, Shmulevich, Sander, and Stuart]{weinstein2013cancer}
John~N Weinstein, Eric~A Collisson, Gordon~B Mills, Kenna~R Shaw, Brad~A Ozenberger, Kyle Ellrott, Ilya Shmulevich, Chris Sander, and Joshua~M Stuart.
\newblock The cancer genome atlas pan-cancer analysis project.
\newblock \emph{Nature genetics}, 45\penalty0 (10):\penalty0 1113--1120, 2013.

\bibitem[Zeitler et~al.(2023)Zeitler, Vlontzos, and Gilligan-Lee]{zeitler2023non}
Jakob Zeitler, Athanasios Vlontzos, and Ciar{\'a}n~Mark Gilligan-Lee.
\newblock Non-parametric identifiability and sensitivity analysis of synthetic control models.
\newblock In \emph{Conference on Causal Learning and Reasoning}, pages 850--865. PMLR, 2023.

\end{thebibliography}

\appendix
\appendixpage

\section{Dataset}
\label{appendix:dataset}

\textbf{Synthetic dataset:} the following (python) pseudo-code describes how data is generated for the Synthetic dataset. The dataset generate for experiments in Sec. \ref{sec:experiments} has 1K samples with 5 causal dimensions and 5 non-causal dimensions.
\begin{lstlisting}[language=Python]
def synthetic_dataset(n: int, y_noise_std: float)
    x = Normal(0, 1, n)
    t_causal = x[:, :causal_dimensions] + Normal(0, 1, n)
    t_non_causal = x[:, causal_dimensions:] + Normal(0, 1, n)
    t = concat(t_causal, t_non_causal, axis=-1)

    mask_t = zeros_like(t)
    mask_t[:, :causal_dimensions] = 1

    y_noise = Normal(0, y_noise_std, n)
    y = sum(mask_t * t + x + y_noise, axis=-1)
    return y
\end{lstlisting}

\textbf{Coat recommender dataset:} the coat recommender dataset is a real-world dataset generated using ratings of users to coats. All components of the treatment/coat are determined causal and additionally we add a 8 dimensional vector to the treatment that is correlated with the covariates/users. The dataset has 10K samples with 33 causal dimensions and 8 non-causal dimensions.

\textbf{Molecule dataset} the molecules dataset is another real-world dataset used in the experiments with 5K samples, 8 causal dimensions and 8 non-causal dimensions. Note that we use the PCA covariates of the dataset and properties as treatments.

\section{Model}
\label{appendix:model}

\textbf{CATE model:} Fig. \ref{fig:model_dag} shows the architecture of the model. The covariates sub-network has one layer and treatment sub-network has two layers of hidden size 32. The common sub-network has two layers of size 64 and 32 each. It is optimized using Adam optimizer, learning rate of $1e-4$ and Huber loss.

\textbf{Contrastive model:} is the same as the CATE model but with a weighted contrastive loss. The triplet loss is weighted by 0.1 for the Synthetic dataset and 1 for the Molecule and Recommender datasets. The margin hyperparameter of the loss is set to 30 for the Synthetic dataset and 100 for the Molecule and Recommender datasets.

\textbf{SIN:} uses the same hyperparameters as in the original paper but with 3 layers and reduced to 32 the hidden size of the sub-networks. Additionally the GNN is replaced with an MLP due to the adaptations of the dataset to have non-causal dimensions.

\section{Compute}
\label{appendix:compute}
All our experiments run on a CPU machine with 16 cores and 64GB of memory.

\end{document}